\documentclass{article}
\usepackage{ijcai16}
\usepackage{times}

\usepackage{amssymb}
\usepackage{graphicx}
\usepackage{amsmath}
\usepackage{amsthm}
\usepackage{algpseudocode}
\usepackage{algorithm}
\usepackage{wrapfig}
\usepackage{dsfont}
\newcommand{\citet}[1]{\citeauthor{#1}~(\citeyear{#1})}

\theoremstyle{plain}
\newtheorem{theorem}{Theorem}
\newtheorem{lemma}{Lemma}

\theoremstyle{definition}

\newtheorem{assumption}{Assumption}
\theoremstyle{remark}

\DeclareMathOperator*{\argmax}{arg\,max}

\def\mathhyphen{{\hbox{-}}}

\title{PAC Greedy Maximization with Efficient Bounds \\ on Information Gain for Sensor Selection}
\author{Yash Satsangi \\
University of Amsterdam \\
y.satsangi@uva.nl
\And 
Shimon Whiteson \\
University of Oxford \\ 
shimon.whiteson@cs.ox.ac.uk 
\And
Frans A. Oliehoek \\
University of Liverpool \\
University of Amsterdam \\
frans.oliehoek@liverpool.ac.uk
}

\begin{document}

\maketitle

\begin{abstract}
\emph{Submodular} function maximization finds application in a variety of real-world decision-making problems. However, most existing methods, based on greedy maximization, assume it is computationally feasible to evaluate $F$, the function being maximized.  Unfortunately, in many realistic settings $F$ is too expensive to evaluate exactly even once. We present \emph{probably approximately correct greedy maximization}, which requires access only to cheap anytime confidence bounds on $F$ and uses them to prune elements.  We show that, with high probability, our method returns an approximately optimal set. We also propose novel, cheap confidence bounds for \emph{conditional entropy}, which appears in many common choices of $F$ and for which it is difficult to find unbipased or bounded estimates. Finally, results on a real-world dataset from a multi-camera tracking system in a shopping mall demonstrate that our approach performs comparably to existing methods, but at a fraction of the computational cost.
\end{abstract}

\section{Introduction}
\nocite{*}

\emph{Submodularity} is a property of set functions that formalizes the notion of \emph{diminishing returns} i.e., adding an element to a set increases the value of the set function by a smaller or equal amount than adding that same element to a subset. Many real-world problems involve maximizing submodular functions, e.g., summarizing text \cite{textsumm1,textsumm2}, selecting subsets of training data for classification \cite{activeLearn1}, or selecting sensors to minimize uncertainty about a hidden variable \cite{satsangi}\footnote{This is an improved version of this paper. Lemma 1 was found to be incorrect in the earlier version (now corrected). We thank Csaba Szepesv\'{a}ri for pointing that.}. 

Formally, given a ground set $\mathcal{X} = \{1 ,2 \dots n \}$, a set function $F : 2^{\mathcal{X}} \to \mathbb{R}$, is submodular if for every $\mathcal{A}_{M} \subseteq \mathcal{A}_{N} \subseteq \mathcal{X}$ and $i \in \mathcal{X} \setminus \mathcal{A}_N$,
\begin{equation}
\Delta_{F}(i|\mathcal{A}_{M}) \geq \Delta_{F}(i|\mathcal{A}_{N}),
\end{equation}
where $\Delta_{F}(i|\mathcal{A}) = F(\mathcal{A} \cup i) - F(\mathcal{A})$ is the \emph{marginal gain} of adding $i$ to $A$. 
Typically, the aim is to find an $\mathcal{A}^{*}$  that maximizes $F$ subject to certain constraints.  Here, we consider a constraint on $\mathcal{A}^{*}$'s size: $\mathcal{A}^{*} = \arg\max_{\mathcal{A} \subseteq \mathcal{X}:|\mathcal{A}| \leq k} F(\mathcal{A})$. 

As $n$ increases, the $n \choose k$ possibilities for $\mathcal{A}^{*}$ grow rapidly, rendering naive maximization intractable. Instead, \emph{greedy maximization} finds an approximate solution $\mathcal{A}^{G}$  faster by iteratively adding to a partial solution the element that maximizes the marginal gain. \citet{nemhauser} showed that the value obtained by greedy maximization is close to that of full maximization, i.e.,
$F(\mathcal{A}^{G}) \geq (1 - e^{-1})F(\mathcal{A}^{*})$, if $F$ is submodular, non-negative and monotone.

\emph{Lazy greedy maximization} \cite{lazy} accelerates greedy maximization by pruning elements whose marginal gain on the last iteration ensures that their marginal gain on the current iteration cannot be maximal.
\emph{Lazier greedy maximization} \cite{lazier} provides further speedup by evaluating the marginal gain only of a randomly sampled subset of elements at each iteration.  Other variations \cite{multistage,thresholdgreedy} also minimize the number of marginal gain computations.

However, these methods assume it is computationally feasible to exactly compute $F$, and thus the marginal gain.  In many settings, this is not the case.  For example, consider a surveillance task \cite{satsangi} in which an agent aims to minimise uncertainty about a hidden state by selecting a subset of sensors that maximise \emph{information gain}. Computing information gain is computationally expensive, especially when the hidden state can take many values, as it involves an expectation over the entropy of posterior beliefs about the hidden state. When surveilling large areas like shopping malls, exactly computing the entropy of a single posterior belief becomes infeasible, let alone an expectation over them.

In this paper, we present a new algorithm called \emph{probably approximately correct greedy maximization}.  Rather than assuming access to $F$ itself, we assume  access only to confidence bounds on $F$. In particular, we assume that these bounds are cheaper to compute than $F$ and are \emph{anytime}, i.e., we can tighten them by spending more computation time, e.g., by generating additional samples. Inspired by lazy greedy maximization, our method uses confidence bounds to prune elements, thereby avoiding the need to further tighten their bounds. Furthermore, we provide a PAC analysis that shows that, with high probability, our method returns an approximately optimal set.

Given an unbiased estimator of $F$, it is possible to use concentration inequalities like Hoeffding's inequality to obtain the confidence bounds needed by PAC greedy maximization. Unfortunately, many applications, such as sensor placement and decision tree induction, require information-theoretic definitions of $F$ such as information gain.  These definitions depend on computing entropy over posterior beliefs, which are impossible to estimate in unbiased way \cite{Paninski}. The absence of an unbiased estimator renders Hoeffding's inequality inapplicable and makes it hard to obtain computationally \emph{cheap} confidence bounds on conditional entropy \cite{IGestimates,racing}. Therefore, in this paper, we propose novel, cheap confidence bounds on conditional entropy.

Finally, we apply PAC greedy maximization with these new confidence bounds to a real-life dataset collected by agents controlling a multi-camera tracking system employed in a shopping mall. Our empirical results demonstrate that our approach performs comparably to greedy and lazier greedy maximization, but at a fraction of the computational cost, leading to much better scalability.

\section{Background}

Given a set function $F:2^{\mathcal{X}} \to \mathbb{R}$, \emph{greedy maximization} \cite{nemhauser}  computes a subset $\mathcal{A}^{G} \subseteq \mathcal{X}$ that approximates $\mathcal{A}^{*} = \arg\max_{\mathcal{A} \in
 \mathcal{A}^{+}} F(\mathcal{A})$, where $\mathcal{A}^{+} = \{\mathcal{A} \subseteq \mathcal{X}: |\mathcal{A}| \leq k\}$. As shown in Algorithm \ref{GreedyAlg}, it does so by repeatedly adding to $\mathcal{A}^{G}$ the element $i$ that maximizes the marginal gain $\Delta_{F}(i|\mathcal{A}^{G})$. Because it is greedy, this method is much faster than naive maximization.

\begin{algorithm}
\caption{$\mathtt{greedy}\mathhyphen\mathtt{max}(F,\mathcal{X},k)$}\label{GreedyAlg}
\begin{algorithmic}
\State $\mathcal{A}^{G} \gets \emptyset$
\For {$m = 1 \ to \ k$}
\State $\mathcal{A}^{G} \gets \mathcal{A}^{G} \cup \argmax_{i \in \mathcal{X} \setminus \mathcal{A}^G}\Delta_{F}(i|\mathcal{A}^{G})$
\EndFor
\State return $\mathcal{A}^{G}$
\end{algorithmic}
\end{algorithm}

\noindent \citet{nemhauser} showed that, under certain conditions, this method has bounded error.
\begin{theorem} (Nemhauser et al., 1978)
If $F$ is non-negative, monotone and submodular%
, then
$F(\mathcal{A}^{G}) \geq (1 - e^{-1}) F(\mathcal{A}^{*})$.
\end{theorem}

\emph{Lazy greedy maximization} \cite{lazy} accelerates greedy maximization by pruning elements whose marginal gain cannot be maximal by maintaining a priority queue of all elements in which each element's priority is its marginal gain computed in the \emph{previous} iteration. If in the current iteration, the marginal gain of the element with the highest priority is higher than the priority of the next element, then the current iteration is terminated since submodularity guarantees that the marginal gain of the remaining elements can only decrease. Lazy greedy maximization computes the same $\mathcal{A}^{G}$ as greedy maximization and is much faster in practice. 

\section{Problem Setting}

In this paper, we consider a variation on submodular function maximization in which evaluating $F$, and therefore the marginal gain, is prohibitively expensive, rendering greedy and lazy greedy maximization inapplicable.  Instead, we assume access to computationally cheap confidence bounds on $F$.  
\begin{assumption} \label{as:zero}
We assume access to anytime upper and lower confidence bounds on $F(\mathcal{A})$ and a  $\mathtt{tighten}(\mathcal{A}, t)$ procedure that for all $\mathcal{A} \in \mathcal{A}^{+}$ that takes in as input arguments $\mathcal{A}$ and $t$ ($t$ is a positive integer) and returns $U_{t}(\mathcal{A})$ and $L_{t}(\mathcal{A})$ such that with probability $1 - \frac{\delta_{l}}{nt(t+1)}$, $L_{t}(\mathcal{A}) \leq F(\mathcal{A})$ and with probability $1 - \frac{\delta_{u}}{nt(t + 1)}$, $U_{t}(\mathcal{A}) \geq F(\mathcal{A})$, for some fixed value of $\delta_l$ and $\delta_u$. Also, we assume that the lower and upper confidence bounds $L_t$ and $U_t$ are monotonically increasing and decreasing respectively, that is,  $L_{t} \leq L_{t'}$ and $U_{t} \geq U_{t'}$ for $t' > t$. (Here $n$ is the size of $\mathcal{X} = \{ 1, 2, \dots, n \}$ and $\mathcal{A}$ is a subset of $\mathcal{X}$ of size less than or equal to $k$. )
\end{assumption}

These assumptions are satisfied in many settings where $F$ is too expensive to compute exactly.  For example, if $F(\mathcal{A}) = \mathbb{E}[X|\mathcal{A}]$ for some random variable $X$, then $\hat{F}(\mathcal{A}) = \frac{1}{N}(\sum_{i=1}^N x_i)$, where the $x_i$'s are i.i.d.\ samples of $X$, is an unbiased estimator of $\hat{f}(\mathcal{A})$ for which $U_t$ and $L_t$ can easily be constructed using, e.g., Hoeffding's inequality. According to Hoeffding's inequality for $x_i \in [0,1]$, 
\begin{equation}  \label{eq:Hoeffding}
\Pr(| F(\mathcal{A}) - \mathbb{E}[\hat{F}(\mathcal{A})]| \geq \epsilon) \leq 2 e^{(-2 \epsilon^{2} N)}.
\end{equation}
Using Hoeffding's inequality, $L_{t}$ and $U_{t}$ can be constructed as: with probability $1 - \frac{\delta_l}{t(t+1)}$, $L_{t}(\mathcal{A}) = \hat{F}(\mathcal{A}) - \sqrt{\frac{1}{2N} \log(\frac{2t(t+1)}{\delta_{l}})} \leq F(\mathcal{A})$ is true and that with probability $1 - \frac{\delta_u}{t(t+1)}$,  $U_{t}(\mathcal{A}) = \hat{F}(\mathcal{A}) + \sqrt{\frac{1}{2N} \log(\frac{2t(t+1)}{\delta_{u}})} \geq F(\mathcal{A})$. Furthermore, $\mathtt{tighten}$ procedure can tighten these lower and upper bounds by spending more computation, thereby, using more samples (higher $N$) to compute $\hat{F}$.
However, we specifically do \emph{not} assume access to an unbiased estimator of $Q$.  Instead, we seek an algorithm that performs submodular function maximization given only $U_{t}$, $L_{t}$, and $\mathtt{tighten}$.

The absence of an unbiased estimator of $F$ arises in many settings in which $F$ is defined using information-theoretic metrics such as \emph{information gain} or \emph{entropy}.  For example, consider the \emph{sensor selection} problem \cite{ApproxDynamProg,Spaan09icaps} in which an agent has a set of sensors $\mathcal{X} = \{1,2 \dots n\}$  giving information about a hidden state $s$. For each sensor $i$, $z_{i}$ denotes the observation the agent will receive if it selects that sensor, with $z_i = \emptyset$ if not selected. $\mathbf{z} = \langle z_{1}, z_{2} \dots z_{n} \rangle$ denotes the  complete observation vector generated by all sensors. 

Upon selecting sensors $\mathcal{A}$ and observing $\mathbf{z}$, the agent can compute a posterior belief using Bayes rule:
\begin{equation}
b_{\mathbf{z}}^{\mathcal{A}}(s) = \frac{\Pr(\mathbf{z}|s,\mathcal{A})b(s)}{\Pr(\mathbf{z}|b,\mathcal{A})},
\end{equation}
where $\Pr(\mathbf{z}|b,\mathcal{A}) = \sum_{s} b(s) \Pr(\mathbf{z}|s,\mathcal{A})$ and $b(s)$ is a prior belief.  The agent aims to minimize its uncertainty about $s$, measured as the \emph{entropy} of $b(s)$: $H_{b}(s) = - \sum_{s} b(s) \log(b(s)).$

Given $b$ and $\mathcal{A}$, the \emph{conditional entropy} is:
\begin{equation}
H_{b}^{\mathcal{A}}(s|\mathbf{z}) = \sum_{\mathbf{z} \in {\Omega}} \Pr(\mathbf{z}|b,\mathcal{A}) H_{b_{\mathbf{z}}^{\mathcal{A}}}(s),
\end{equation}
where ${\Omega}$ is the set of all possible values of $\mathbf{z}$ that can come from sensors present in the set $\mathcal{A}$.  The agent's goal is to find $\mathcal{A}^{*}$ that maximizes \emph{information gain}:
\begin{equation} \label{eq:ig}
IG_{b}(\mathcal{A}) = H_{b}(s) - H_{b}^{\mathcal{A}}(s|\mathbf{z}).
\end{equation}
Since the first term in (\ref{eq:ig}) is independent of $\mathcal{A}$, we equivalently define $F(\mathcal{A})$ as: 
\begin{equation} \label{eq:obj}
F(\mathcal{A}) = -\sum_{\mathbf{z} \in \Omega} \Pr(\mathbf{z}|b,\mathcal{A}) H_{b_{\mathbf{z}}^{\mathcal{A}}}(s).
\end{equation}

Unfortunately, when there are many possible states and actions, computing $H_{b^{\mathcal{A}}_{\mathbf{z}}}$(s) is not only intractable, but also difficult to efficiently estimate \cite{Paninski,IGestimates,Schurmann}. In fact, \citet{Paninski} showed that no unbiased estimator for entropy exists.  %

Therefore, in the next section we propose a new fundamentally different method that requires only  $U$, $L$, and $\mathtt{tighten}$.  To solve sensor selection in particular, we also need cheap anytime implementations of $U$ and $L$ for conditional entropy, which we propose in Section \ref{sec:confBounds}.

\section{Method}

In this section, we propose \emph{probably approximately correct greedy maximization}, which enables an agent to perform submodular function maximization without ever computing $F$ exactly. The main idea is to use $U$ and $L$ to prune elements that with high probability do not maximize marginal gain.

Our approach is inspired by lazy greedy maximization.  To see how, it is helpful to view lazy greedy maximization as a pruning method: terminating an iteration before the priority queue is empty effectively prunes each element whose upper bound (given by its marginal gain on the previous iteration) is lower than the maximum lower bound (given by the best marginal gain found so far on the current iteration).  

PAC greedy maximization generalizes this idea in two ways. First, it accepts arbitrary upper and lower bounds.  This makes it possible to replace the bounds used by lazy greedy maximization, which rely on exact computation of marginal gain, with cheaper ones.  Second, it uses confidence bounds instead of hard bounds.  By tolerating a small probability of error, our approach can prune more aggressively, enabling large speedups while maintaining a PAC bound.

\begin{algorithm}
\caption{$\mathtt{pac}\mathhyphen\mathtt{greedy}\mathhyphen\mathtt{max}(U,L,\mathcal{X},k,\epsilon_1,t)$}\label{alg:pac-gm}
\begin{algorithmic}
\State $\mathcal{A}^{P} \gets \emptyset$
\For{$m = 1 \ to \ k$}
\State $\mathcal{A}^{P} \gets \mathcal{A}^{P} \cup \mathtt{pac}\mathhyphen\mathtt{max}(\mathtt{tighten},\mathcal{X},\mathcal{A}^P,\epsilon_1)$
\EndFor
\State return $\mathcal{A}^{P}$
\end{algorithmic}
\end{algorithm}

Algorithm \ref{alg:pac-gm} shows the main loop, which simply adds at each iteration the element selected by the $\mathtt{pac}\mathhyphen\mathtt{max}$ subroutine. Algorithm \ref{alg:pac-m} shows this subroutine, which maintains a queue of unpruned elements prioritized by their upper bound.  In each iteration of the outer while loop, $\mathtt{pac}\mathhyphen\mathtt{max}$ examines each of these elements and prunes it if its upper bound is not at least $\epsilon_1$ greater than the max lower bound found so far.  In addition, the element with the max lower bound is never pruned. If an element is not pruned, then its bounds are tightened.  Algorithm \ref{alg:pac-m} terminates when only one element remains or when the improvement produced by tightening $U$ and $L$ falls below a threshold $t$.

\begin{algorithm}
\caption{$\mathtt{pac}\mathhyphen\mathtt{max}(\mathtt{tighten},\mathcal{X},\mathcal{A}^P,\epsilon_1)$}\label{alg:pac-m}
\begin{algorithmic}[1]
\State \hspace{-10mm} \Comment{Input: $\mathtt{pac}\mathhyphen\mathtt{max}$ takes as input access to $\mathtt{tighten}$ procedure; $\mathcal{X} = \{1, 2, \dots, n \}$ original set of $n$ elements; $\mathcal{A}^P$ a subset of $\mathcal{X}$, in this case, $\mathcal{A}^P$ is the partial solution maintained by $\mathtt{pac}\mathhyphen\mathtt{greedy}\mathhyphen\mathtt{max}$, $\epsilon$ a positive real number. }
\vspace{2mm}
\State $i^P \gets 0$       \hspace{10mm}           \Comment{element with max lower bound}
\State $\rho \gets$ empty priority queue 
\State $t \gets 0$   \hspace{30mm} \Comment{ $t$ is the iteration number.} 
\vspace{2mm}
\State $t = t + 1$
\For{$i \in \mathcal{X} \setminus \mathcal{A}^P$}
\State $U_{t}(i), L_{t}(i) \gets \mathtt{tighten}(\mathcal{A}^{P} \cup i, t)$  \\ \\ \Comment{ Here $U_{t}(i)$ and $L_{t}(i)$  denote the  upper and lower  bound on  $F(\mathcal{A}^P \cup i)$.} \\
\vspace{2mm}
\State $\rho.\mathtt{enqueue}(i,U_{t}(i))$         
\State $i^P \gets \argmax_{j\in\{i,i^P\}} L_{t}( j)$  
\EndFor
 \While{$(\rho.\mathtt{length}() > 1)$}
 \State $\rho' \gets$ empty priority queue
 \State $t = t + 1$
 \State $i^P$-in-queue = True   \\ \Comment{Set flag to check if $i^P$ is still in $\rho$}
 \While{$\neg\rho.\mathtt{empty}()$}
 \State $i \gets \rho.\mathtt{dequeue}()$   
 \If{$i = i^P$} \\
 \State $i^P$-in-queue = False
 \EndIf
 \vspace{1mm}
 \If{$(i = i^P) \vee (U_{t}(i) \geq L_{t}(i^P) + \epsilon_1)$} 
 \vspace{2mm}
 \State $U_{t}(i), L_{t}(i) \gets \mathtt{tighten}(\mathcal{A}^{P} \cup i, t)$  
 \State $i^P \gets \argmax_{j\in\{i,i^P\}} L_{t}(j)$  
 \State $\rho'.\mathtt{enqueue}(i,U_{t}(i))$
 \ElsIf{$i^P\mbox{-in-queue = True}$}
 \State Continue
 \Else 
 \State Break Inner While Loop
 \EndIf
 \EndWhile
 \State $\rho \gets \rho'$
 \EndWhile
 \State return $i^P$
\end{algorithmic}
\end{algorithm}

Algorithm \ref{alg:pac-m} is the closely related to best-arm identification algorithms \cite{bestarm} for multi-armed bandits. Spcifically, \ref {alg:pac-m} is closes to the algorithm Hoeffding's races, presented in \cite{Maron94,Maron97} except that \cite{Maron94,Maron97} propose explicitly to use Hoeffding's inequality to compute and tighten the upper and lower confidence bound\footnote{The other \emph{minute} differences between $\mathtt{pac}\mathhyphen\mathtt{max}$ and Hoeffding's races are (a) the use of priority queue in $\mathtt{pac}\mathhyphen\mathtt{max}$, and (b) that Hoeffding's races do(es) not \emph{explicitly} take into account the number of times $\mathtt{tighten}$ procedure was previously called as an input parameter to the $\mathtt{tighten}$ procedure.}. Consequently, the analysis and convergence of the algorithms that they present are reliant on the application of Hoeffding's inequality and, thus, are applicable only for functions that can be estimated in an unbiased manner. This is in contrast to Algorithm \ref{alg:pac-m} and its analysis that is given in the later section, both of which do not necessarily make any assumption on the way in which the upper and lower confidence bounds are generated or tightened.

\section{PAC Bounds}
In this section, we analyze PAC greedy maximization. With oracle access to $F$, greedy maximization is guaranteed to find $\mathcal{A}^{G}$ such that $F(\mathcal{A}^{G}) \geq (1 - e^{-1})F(\mathcal{A^{*}})$, if $F$ is monotone, non-negative and submodular \cite{nemhauser}. Since PAC greedy maximization does not assume oracle access to $F$ and instead works with cheap anytime confidence bounds on $F$, we prove a PAC bound for PAC greedy maximization. In particular, we prove that under the same conditions, PAC greedy maximization finds a solution $\mathcal{A}^{P}$ such that, with high probability, $F(\mathcal{A}^{P})$ is close to $F(\mathcal{A}^{*})$.

We can now prove a lemma that shows that, with high probability, the marginal gain of the element picked by $\mathtt{pac}\mathhyphen\mathtt{max}(U,L,\mathcal{A},\epsilon_1)$ is at least nearly as great as the average marginal gain of the elements not in $\mathcal{A}$.

\begin{lemma}  \label{lem:pacLemma}
Let $\mathcal{X} = \{1, 2, \dots, n \}$, $\mathcal{A}^+ = \{\mathcal{A} \subseteq \mathcal{X} : |\mathcal{A} | \leq k \}$, and $F: 2^{\mathcal{X}} \to \mathbb{R}_{+}$. If $\mathtt{pac}\mathhyphen\mathtt{max}(\mathtt{tighten},\mathcal{X}, \mathcal{A},\epsilon_1)$ terminates and returns $i^{P}$, and if Assumption \ref{as:zero} holds, then with probability (at least) $1 - \delta$, 
\begin{equation}
F(\mathcal{A}^P \cup i^P) \geq F(\mathcal{A}^P \cup i^*) - \epsilon_1,
\end{equation}
where $\delta_{1} = \delta_u + \delta_{l}$ and  $i^* = \arg\max_{i \in \mathcal{X} \setminus \mathcal{A}^{P}} F(\mathcal{A}^{P} \cup i)$ and $\mathcal{A}^{P}$ is any set in $\mathcal{A}^+$.
\end{lemma}

\begin{proof}
If $\mathtt{pac}\mathhyphen\mathtt{max} $ returns $i^P = i^*$, then the Lemma holds trivially, since 
\begin{equation}
F(\mathcal{A}^P \cup i^*) \geq F(\mathcal{A}^P \cup i^*) - \epsilon_1. 
\end{equation}

For the case that $\mathtt{pac}\mathhyphen\mathtt{max}$ returns $i^P \neq i^*$, we provide the proof here.

Lets assume that $\mathtt{pac}\mathhyphen\mathtt{max}$ returns $i^P$ after $T$ ($T$ not known or fixed) total iterations. That is, $t$ goes from 0 to $T$.

We prove this Lemma in two parts: 
\begin{itemize}
\item In part A, we show that if the assumed confidence intervals $U_{t}(i)$ and $L_{t}(i)$ hold for all $t$ and $i$, then $\mathtt{pac}\mathhyphen\mathtt{max}$ returns an $\epsilon$-optimal element. 
That is, if 
\begin{equation}
U_{t}(i) \geq F(\mathcal{A}^P \cup i)    
\end{equation}
and
\begin{equation}
L_{t}(i) \leq F(\mathcal{A}^P \cup i)
\end{equation}
is true for all $i \in \mathcal{X}$ and $t \in \{1, 2, \dots, T\}$, then
\begin{equation}
F(\mathcal{A}^P \cup i^P) \geq F(\mathcal{A}^P \cup i^*) - \epsilon_1. 
\end{equation}
\item In part B, we compute the probability that the confidence intervals hold for all $i$ and $t$.

We show that this probability is lower bounded by $1 - \delta_{1}$ if Assumption \ref{as:zero} holds. Here $\delta_1 = \delta_l + \delta_u$ is the probability that the confidence intervals (UCI or LCI) are not true at least once in $T$ iterations for at least one $i$. 
\end{itemize}

\textbf{Part A}: To show, if for all $i$ and $t$, confidence intervals hold, that is, $U_{t}(i) \geq F(\mathcal{A}^{P} \cup i) \geq L_{t}(i)$ is true for all $i,  t$, then 
\begin{equation}
F(\mathcal{A}^P \cup i^P) \geq F(\mathcal{A}^P \cup i^*) - \epsilon_1.
\end{equation}

At any iteration $t \in \{1, 2, \dots, T \}$ $\mathtt{pac}\mathhyphen\mathtt{max}$ maintains the element with max lower bound. Lets denote the element with max lower bound at the end of iteration $t$ by $i^{P}_{t}$.  Since $i^*$ was eliminated ($i^P \neq i^*$), thus at some iteration $t'$, its upper bound was lower than the maximum lower bound + $\epsilon_1$ (lets say of element $i^P_{t'}$). Let $L_{t'}(i)$ denote the lower bound (and $U_{t'}(i)$ denote the upper bound) at iteration $t'$ of element $i$, then, the lower bound of element $i^P_{t'}$ is greater than the upper bound of $i^*$ minus $\epsilon$ at some iteration $t'$: 
\begin{equation}  \label{pr7:eq0}
L_{t'}(i^P_{t'}) \geq  U_{t'}(i^*) - \epsilon_1. 
\end{equation}

Since (a) we have assumed that $L_{t}$ is monotonically increasing, and (b) $\mathtt{pac}\mathhyphen\mathtt{max}$ returns $i^P$, this implies on termination the element with maximum lower bound is $i^P$, and this lower bound on $i^P$ has to be greater than $L_{t'}(i^P_{t'})$, since $i^P$ was able to replace $i^P_{t'}$ at an iteration $t > t'$. 
\begin{equation}  \label{pr7:eq2}
L_{T}(i^P)  \geq L_{t'}(i^P_{t'})
\end{equation}

Combining \eqref{pr7:eq0}, and \eqref{pr7:eq2}, we get, 
\begin{equation}
L_{T}(i^P) \ge L_{t'}(i^P_{t'}) \ge U_{t'}(i^*) - \epsilon_1
\end{equation}

If confidence interval hold for all $t$ and $i$, then $ U_{t'}(i^*) \geq F(\mathcal{A}^P \cup i^*)$  and $F(\mathcal{A}^P \cup i^P) \geq L_{T}(i^P)$, this implies, 
\begin{equation}
F(\mathcal{A}^P \cup i^P) \ge L_{T}(i^P) \ge U_{t'}(i^*) - \epsilon_1 \ge F(\mathcal{A}^P \cup i^*) - \epsilon_1 .
\end{equation}

Thus, if confidence intervals $U_{t}(i)$ and $L_{t}(i)$ hold for all $t$ and $i$, then, 
\begin{equation}
F(\mathcal{A}^P \cup i^P) \geq F(\mathcal{A}^{P} \cup i^*) - \epsilon.
\end{equation}

\textbf{Part B}: In part B, we compute the probability that the upper and lower confidence intervals hold for all $t$ and $i$. The reasoning in this part follows from the proof presented in \cite{Maron94,Maron97}. 

We will be using extensively the union bound during this part of the proof. According to union bound the probability of the union of events $A_1, A_2, \dots, A_l$ is bounded by the sum of their individual probabilities: 
\begin{equation}  \label{eq:unionBound}
\Pr(A_1 \vee A_2 \vee \dots \vee A_l) = \Pr( \bigcup_{l} A_l) \leq \sum_{l} \Pr(A_l) 
\end{equation}

To compute the probability that the confidence intervals hold for all $i$ for all $t$, we observe that this probability is equal to 1 - the probability that the confidence intervals do not hold for at least one $i$ during at least one iteration $t$. So we want to compute the probability:
$\Pr$({\small{upper confidence interval (UCI) OR lower}}  \small{confidence interval (LCI) do not hold for at least one $i$ for at least one value of $t$ }).

To compute this probability, lets start with the probability of the confidence interval to NOT hold for one particular $i = i'$ at one particular iteration $t = t'$. We have assumed that at iteration $t$, $\mathtt{tighten}(\mathcal{A}^P \cup i, t)$ returns $U_{t}(i)$ (and $L_{t}(i)$) such that with probability $1 - \frac{\delta_u}{n t(t + 1)}, U_{t}(i) \geq F(\mathcal{A}^P \cup i)$ (this condition means that upper confidence interval holds) is true. This implies that for a particular $i = i'$ at iteration $t = t'$, the probability of upper confidence interval to not hold is (less than) $\frac{\delta_u}{n t'(t' + 1)} $ and the probability that lower confidence interval ($L_{t}(i) \leq F(\mathcal{A}^P \cup i)$) does not hold is (less than) $\frac{\delta_l}{nt'(t'+1)}$. 

By Assumption \ref{as:zero}, 
\begin{equation}
\Pr(\mbox{UCI is not true for $i'$ at iteration $t'$ }) \leq \frac{\delta_u}{n t'(t' + 1)}, 
\end{equation}
and
\begin{equation}
\Pr(\mbox{LCI is not true for $i'$ at iteration $t'$ }) \leq \frac{\delta_l}{n t'(t' + 1)}, 
\end{equation}

Thus, using union bound,
\begin{align}  \label{eq:oneionet}
\Pr\Big(\mbox{ \small UCI OR LCI is not true for $i'$ at iteration $t'$ } \Big) \nonumber \\ \leq \frac{\delta_u + \delta_l}{n t' (t' + 1)}.
\end{align}

Again using union bound, probability that confidence intervals do NOT hold for $i'$  at $t=1$ OR $t=2$ OR $t=3$ OR $\dots$ OR $t=T$ is bounded by sum of individual (probability that confidence intervals do NOT hold for $i'$ for $t = 1$) + (probability that confidence intervals do NOT hold for $i'$ for $t = 2) +  \dots$ (series ends at $t=T$).

From equation \eqref{eq:oneionet}, we know the probability that confidence intervals do not hold for $i'$ at iteration $t'$ is less than $\frac{\delta_u + \delta_l}{nt' (t' + 1)}$. 
\begin{align}
\Pr({\small(\mbox{UCI or LCI is not true for $i'$ at least once in } t \in \{1, 2, \dots, T\} }) \nonumber  \\ \leq \sum_{t= 1}^{T}  \frac{\delta_l + \delta_u}{nt(t+1)}
\end{align}

The sum over $t$ of the series $\frac{1}{t(t+1)}$, that is $S_{T} = \sum_{t=1}^{T} \frac{1}{t(t+1)}$ is bounded by $(1 - \frac{1}{T+1})$ for a finite $T$ and even as $\lim{T \to \infty}$, $\lim_{T \to \infty} \sum_{t=1}^{T} \frac{1}{t(t+1)}$ is bounded by 1\footnote{$\sum_{t=1}^{T} \frac{1}{t(t+1)}$ can be expressed as: 
\begin{align}
\sum_{t=1}^{T} \frac{1}{t(t+1)} &= \sum_{t=1}^{T} [\frac{1}{t} - \frac{1}{t+1}] \hfill \\
&= [1 - \frac{1}{2} + \frac{1}{2} - \frac{1}{3} + \frac{1}{3} - \dots - \frac{1}{T+1}] \\
& = [1 - \frac{1}{T+1}].
\end{align} 
}.

Thus, using union bound,
\begin{align}
\Pr({\small(\mbox{UCI or LCI is not true for $i'$ at least once in } t \in \{1, 2, \dots, T\} }) \nonumber \\ \leq \frac{\delta_u + \delta_l}{n}
\end{align}

Again we can use union bound to show that the probability that the confidence intervals do not hold for $i=1$ OR $i=2$ OR $\dots$ OR $i=n$, at least once in $t \in \{1, 2, \dots, T\}$ is bounded by the (probability that the confidence interval do not hold for  $i=1$ at least once in $t \in \{1, 2, \dots, T\}$) + (probability that the confidence interval do not hold for  $i=2$ at least once in $t \in \{1, 2, \dots, T\}$) + $\dots$ + (probability that the confidence interval do not hold for  $i=n$ at least once in $t \in \{1, 2, \dots, T\}$) . 

Since for each $i$ the probability that the confidence interval do not hold for $i$ at least once in $t \in \{1, 2, \dots, T\}$ is bounded by $\frac{\delta_u + \delta_l}{n}$. Taking the sum over $n$ terms yields: 
\begin{align}
\Pr({\scriptsize \mbox{UCI or LCI is not true for at least one $i$ at least once in } t \in \{1, 2, \dots, T\} }) \nonumber \\ \leq \delta_u + \delta_l.
\end{align}

Finally, since probability that confidence intervals hold for all $i$ for all $t \in \{1, 2, \dots, T \}$ = 1 - probability confidence intervals do not hold at least for one $t \in \{1, 2, \dots, T \}$ for at least one $i$, we can write that with probability $1 - \delta_1$, 
\begin{equation}
F(\mathcal{A}^P \cup i^P) \geq F(\mathcal{A}^P \cup i^*) - \epsilon_1.
\end{equation}
\end{proof}

Next, we show that, if in each iteration of greedy maximization an $\epsilon$-optimal element is returned with probability $1- \delta$, then greedy maximization returns a set that is $k\epsilon$-optimal with probability $1 - k\delta$, where $k$ is the number of iteration greedy maximization is run for. 

\begin{theorem} \label{th:main1}
Let $\mathcal{X} = \{1, 2, \dots n \}$,  $\mathcal{A}^{+} : \{ \mathcal{A} \subseteq \mathcal{X} : |\mathcal{A}| \leq k \}$, and $F : 2^{\mathcal{X}} \to \mathbb{R}_{+}$ be a non-negative, monotone and submodular in $\mathcal{X}$. if Assumption \ref{as:zero} holds and if $\mathtt{pac}\mathhyphen\mathtt{max}$ terminates every time it is called then, with probability $ 1 - \delta$, 
\begin{equation} \label{eq:thmain}
F(\mathcal{A}^{P}) \geq (1 - e^{-1}) F(\mathcal{A}^{*}) - \epsilon,
\end{equation}
where $\mathcal{A}^{P} = \mathtt{pac}\mathhyphen\mathtt{greedy}\mathhyphen\mathtt{max}(\mathtt{tighten},\mathcal{X},k, \epsilon_1)$, $\mathcal{A}^{*} = \arg\max_{\mathcal{A} \in \mathcal{A}^{+}} F(\mathcal{A})$, $\delta = k \delta_{1}$, and $\epsilon = k \epsilon_1$. (Here $\delta_1 = \delta_l + \delta_u$, and $\delta_l$ and $\delta_u$ are defined in Assumption \ref{as:zero}.)
\end{theorem}
\begin{proof}
Let $\mathcal{A}^{P}_{m}$ denote the subset returned by $\mathtt{pac}\mathhyphen\mathtt{greedy}\mathhyphen\mathtt{max}$ after $m$ iterations, that is, $\mathcal{A}^{P}_{m} = \mathtt{pac}\mathhyphen\mathtt{greedy}\mathhyphen\mathtt{max}(\mathtt{tighten}, \mathcal{X}, m,\epsilon_1)$ and let $\{i_{1}^*, i_{2}^*, \dots i_{k}^{*} \}$ (arbitrary order),  be the $k$ elements of $\mathcal{A}^*$. 
We denote the marginal gain of adding $i$ to a subset $\mathcal{A}$ as: 
\begin{equation}  \label{eq:marginalGain}
\Delta_{F}(i|\mathcal{A})  = F(\mathcal{A} \cup i) - F(\mathcal{A}). 
\end{equation}

\vspace{5mm}

To prove Theorem \ref{th:main1}, we first prove an intermediate result that we will use later in the proof: 
Starting with the statement of Lemma \ref{lem:pacLemma}, with probability $1 - \delta_{1}$, \
\begin{equation}
F(\mathcal{A}^P_{m} \cup i^P) \geq F(\mathcal{A}^P_{m} \cup i^*) - \epsilon_{1}, 
\end{equation}
where $i^* = \arg\max_{i \in \mathcal{X} \setminus \mathcal{A}^{P}_m} F(\mathcal{A}^{P}_m \cup i)$ and $i^P = \mathtt{pac}\mathhyphen\mathtt{max}(\mathtt{tighten}, \mathcal{X}, \mathcal{A}^P_{m}, \epsilon_1)$.

This implies the following set of inequalities the explanation of which is provided after them: with probability $1 - \delta_{1}$, 
\begin{align}
&F(\mathcal{A}^P_{m} \cup i^P) \geq F(\mathcal{A}^P_{m} \cup i^*) - \epsilon_{1}  \label{th12:p1:e1}\\ 
&F(\mathcal{A}^P_{m} \cup i^P) - F(\mathcal{A}^{P}_{m}) \geq Q(\mathcal{A}^P_{m} \cup i^*) -Q(\mathcal{A}^P_{m}) - \epsilon_{1} \label{th12:p1:e2}\\ 
&\Delta_{F}(i^P | \mathcal{A}^P_{m}) \geq \Delta_{F}(i^* |\mathcal{A}^P_{m}) - \epsilon_{1} \label{th12:p1:e3} \\
& \Delta_{F}(i^P | \mathcal{A}^P_{m}) \geq \frac{1}{|\mathcal{A}^* \setminus \mathcal{A}^P_{m}|}  \sum_{i \in \mathcal{A}^* \setminus \mathcal{A}^P_{m}} \Delta_{F}(i | \mathcal{A}^P_{m}) - \epsilon_{1} \label{th12:p1:e4} \\
& \Delta_{F}(i^P | \mathcal{A}^P_{m}) \geq \frac{1}{k} \sum_{i \in \mathcal{A}^* \setminus \mathcal{A}^P_{m}} \Delta_{F}(i | \mathcal{A}^P_{m}) - \epsilon_{1}  \label{th12:p1:e5}. 
\end{align}

Eq. \eqref{th12:p1:e1} follows from Lemma \ref{lem:pacLemma}, Eq. \eqref{th12:p1:e2} is simple subtraction of $Q(\mathcal{A}^P_{m})$ from both sides of inequality, Eq. \eqref{th12:p1:e3} is re-writing \eqref{th12:p1:e2} by using the definition of marginal gain as given in \eqref{eq:marginalGain} (and in Chapter 2), Eq. \eqref{th12:p1:e4} is true because $\Delta_{F}(i^* |\mathcal{A}^P_{m})$ is the maximum value of $\Delta_{F}( i |\mathcal{A}^P_{m})$ for all $i \in \mathcal{X} \setminus \mathcal{A}^P_{m}$. This implies it is definitely bigger than the average value of $\Delta_{F}(i|\mathcal{A}^P_{m})$ taken over $\mathcal{A}^* \setminus \mathcal{A}^P_{m}$ where $\mathcal{A}^*$ is a subset of $\mathcal{X}$. Eq. \eqref{th12:p1:e5} is true because $|\mathcal{A}^*| \leq k$.

\vspace{5mm}

The rest of the proof follows the same logic as the proof presented in \cite{submodularity} for Nemhauser's original result on greedy maximization of submodular functions.

We present the following sets of inequalities and then provide the explanations for them below: 
\begin{align}
F(\mathcal{A}^*) & \leq F(\mathcal{A}^* \cup \mathcal{A}^{P}_{m})  \label{th12:eq1} \\
&= F(\mathcal{A}^P_{m}) + \sum_{j = 1}^{k} \Delta_{F}(i^{*}_{j} | \mathcal{A}^P_{m} \cup \{i_{1}^{*}, i_{2}^*, \dots, i_{j-1}^* \} ) \label{th12:eq2}  \\
& \leq F(\mathcal{A}^P_{m}) + \sum_{i \in \mathcal{A}^* \setminus \mathcal{A}^P_{m}} \Delta_{F} (i | \mathcal{A}^{P}_{m}) \label{th12:eq3} 
\end{align}
Equation \eqref{th12:eq1} follows from monotonicity of $Q$, Eq. \eqref{th12:eq2} is a straightforward telescopic sum, Eq. \eqref{th12:eq3} is true because $Q$ is submodular.
\vspace{5mm}

Eq. \eqref{th12:p1:e5} says that with probability $1 - \delta_{1}$, 
\begin{equation} \label{eq:p1:e5p2}
k (\Delta_{F}(i^P | \mathcal{A}_{m}^P) + \epsilon_{1}) \geq \sum_{i \in \mathcal{A}^* \setminus \mathcal{A}^P_{m}} \Delta_{F} (i | \mathcal{A}^{P}_{m}).
\end{equation}
Using \eqref{eq:p1:e5p2}, \eqref{th12:eq3} can be written as: 

With probability $1 - \delta_{1}$,
\begin{align}
F(\mathcal{A}^*) & \leq F(\mathcal{A}^P_{m}) + k ( \Delta_{F}(i^P | \mathcal{A}^P_{m}) + \epsilon_1) \label{th12:eq4} \\
& \leq F(\mathcal{A}^P_{m}) + k(F(\mathcal{A}^P_{m} \cup i^P) - F(\mathcal{A}^P_{m}) + \epsilon_{1}) \label{th12:eq5} \\
& \leq F(\mathcal{A}^P_{m}) + k (F(\mathcal{A}^P_{m+1}) - F(\mathcal{A}^P_{m}) + \epsilon_{1} )  \label{th12:eq6} 
\end{align}

 Eq. \eqref{th12:eq4} follows from \eqref{th12:eq3}, (we just replaced $\sum_{i \in \mathcal{A}^* \setminus \mathcal{A}^P_{m}} \Delta_{Q} (i | \mathcal{A}^{P}_{m})$ with a greater quantity $k ( \Delta_{Q}(i^P | \mathcal{A}^P_{m}) + \epsilon_1)$). Eq \eqref{th12:eq5} follows from the definition of marginal gain in \eqref{eq:marginalGain} and Eq. \eqref{th12:eq6} is true because $\mathcal{A}^P_{m} \cup i^P$ is $\mathcal{A}^P_{m+1}$ by definition of $\mathcal{A}^P_{m}$ and $i^P$ at the start of the proof. 

Lets define $\beta_{m} = F(\mathcal{A}^*) - F(\mathcal{A}^P_{m})$, then \eqref{th12:eq6} can be written as: with probability $1 - \delta_{1}$, 
\begin{align}
F(\mathcal{A}^*) - F(\mathcal{A}^P_{m}) & \leq \nonumber \\ k[ F(\mathcal{A}^*) - F(\mathcal{A}_{m}^{P}) & - (F(\mathcal{A}^*) - F(\mathcal{A}^P_{m+1}) + \epsilon_{1}]  \label{th12:eq7} \\
\beta_{m} & \leq k [\beta_m - \beta_{m+1} + \epsilon_{1}]  \label{th12:eq8} \\ 
\beta_{m+1} & \leq \beta_m (1 - \frac{1}{k}) + \epsilon_{1} \label{th12:eq9} .
\end{align}

Substituting $m = 0$ in \eqref{th12:eq9} gives, with probability $1 - \delta_{1}$, 
\begin{equation}  \label{rec:eq1}
\beta_1 \leq (1 - \frac{1}{k}) \beta_0 + \epsilon_1
\end{equation}

Substituting $m = 1$ in \eqref{th12:eq9} again gives, with probability $1 - \delta_{1}$, 
\begin{equation}  \label{rec:eq2}
\beta_{2} \leq (1 - \frac{1}{k}) \beta_{1} + \epsilon_{1}
\end{equation}

Now combining \eqref{rec:eq1} and \eqref{rec:eq2} and using union bound as presented in equation \eqref{eq:unionBound} (explained below): with probability $1 - 2 \delta_{1}$, 
\begin{align}
\beta_{2} & \leq (1 - \frac{1}{k}) \Big[ (1 - \frac{1}{k}) \beta_{0} + \epsilon_{1} \Big] + \epsilon_{1}   \label{rec:eq3}\\
&\leq (1 - \frac{1}{k})^{2} \beta_0 + (2 - \frac{1}{k}) \epsilon_{1} \label{rec:eq4}\\
& \leq (1 - \frac{1}{k})^2 \beta_0 + 2 \epsilon_1.  \label{rec:eq5} 
\end{align}

The logic behind using union bound here is that the inequality in \eqref{rec:eq1} can fail with probability $\delta_{1}$ and the inequality in \eqref{rec:eq2} can fail with probability $\delta_1$. If both the inequalities in \eqref{rec:eq1} and \eqref{rec:eq2} do not fail then \eqref{rec:eq3} (and consequently \eqref{rec:eq5} is definitely true. The probability that either of inequality in \eqref{rec:eq1} or \eqref{rec:eq2} fails is bounded by their sum of the probabilities that either inequality fails individually: $\delta_{1} + \delta_{1}$. The probability of both inequality in \eqref{rec:eq1} and \eqref{rec:eq2} is true is $1 - 2 \delta_1$. 

Substituting $m = 2$ in \eqref{th12:eq9} again gives, with probability $1 - \delta_{1}$, 
\begin{equation}  \label{rec:eq6}
\beta_{3} \leq (1 - \frac{1}{k}) \beta_{2} + \epsilon_{1}
\end{equation}

Combining \eqref{rec:eq6} with \eqref{rec:eq5}, and using union bound (we just explained how to use union bound here in the paragraph above \eqref{rec:eq6}), with probability $ 1 - 3 \delta_{1}$, 
\begin{align}
\beta_{3} & \leq (1 - \frac{1}{k}) \beta_{2} + \epsilon_{1} \label{rec:eq7} \\
& \leq (1 - \frac{1}{k}) \Big[ (1 - \frac{1}{k})^2 \beta_0  + 2 \epsilon_{1} \Big] + \epsilon_1\label{rec:eq8} \\
& \leq (1 - \frac{1}{k})^{3} \beta_{0} + 3 \epsilon_{1}. \label{rec:eq9}
\end{align}

Continuing like this for $m = 0$ to $k-1$, we get, with probability $1 - k \delta_{1}$, 
\begin{equation}
\beta_{k} \leq (1 - \frac{1}{k})^{k} \beta_{0} + k \epsilon_{1}
\end{equation}

Now using the inequality that $1 - x \leq e^{-x}$ for all $x \in \mathbb{R}$, we get $1 - \frac{1}{k} \leq e^{\frac{-1}{k}}$, which implies, with probability $1 - k \delta_{1}$,
\begin{equation}
\beta_{k} \leq e^{\frac{-k}{k}} \beta_{0} + k \epsilon_{1}
\end{equation}

Using definition of $\beta_{k} = F(\mathcal{A}^*) - F(\mathcal{A}^P) $ and $\beta_{0} = F(\mathcal{A}^*) - F(\mathcal{A}^P_{0})$, with probability $1 - k \delta_{1}$, 
\begin{equation}
 F(\mathcal{A}^*) - F(\mathcal{A}^P) \leq (e^{-1}) [F(\mathcal{A}^*) - F(\mathcal{A}^P_{0})] + k\epsilon_{1}
\end{equation}
Since $F(\mathcal{A}^P_{0}) > 0$, with probability $1 - k \delta_{1}$, 
\begin{align} 
& F(\mathcal{A}^*) - F(\mathcal{A}^P) \leq ( e^{-1}) [F(\mathcal{A}^*)] + k\epsilon_{1} \\
 & F(\mathcal{A}^P) \geq (1 - e^{-1}) F(\mathcal{A}^*) - k\epsilon_{1}
\end{align}

\end{proof}

Theorem \ref{th:main1} proves that PAC greedy maximization, while assuming access only to anytime confidence bounds on $F$, computes $\mathcal{A}^{P}$ such that with high probability $F(\mathcal{A}^{P})$ has bounded error with respect to $F(\mathcal{A}^{*})$. As PAC greedy maximization requires access to cheap upper and lower confidence bounds, in the next section, we propose such bounds for conditional entropy.

\section{Conditional Entropy Bounds} \label{sec:confBounds}

In many settings, $U$ and $L$ can easily be constructed using, e.g., Hoeffding's inequality \cite{Hoeffding63} and $\mathtt{tighten}$ need only fold more samples into an estimate of $F$.  However, Hoeffding's inequality only bounds the error between the estimate and the expected value of the estimator. This in turn bounds the error between the estimate and the true value only if the estimator is unbiased, i.e., the expected value of the estimator equals the true value.

 We are interested in settings such as sensor selection, where $F$ is based on conditional entropy, which is computed by approximating the entropy over a posterior belief. This entropy cannot be estimated in an unbiased way \cite{Paninski}.  Therefore, in this section, we propose novel, cheap confidence bounds on conditional entropy.

We start by defining the maximum likelihood estimate of entropy. Given $M$ samples, $\{s^1, s^{2} \dots s^M\}$ from a discrete distribution $b(s)$, the \emph{maximum likelihood estimator} (MLE) of $b(s)$ is:
\begin{equation} \label{eq:mleent}
\hat{b}(s) = \frac{1}{M} \sum_{j = 1}^{M} \mathds{1}(s^{j},s),
\end{equation}
where $\mathds{1}(s^{j},s)$ is an indicator function that is 1 if $s^{j} = s$ and 0 otherwise.
The MLE of entropy is:

\begin{equation}
{H}_{\hat{b}}(s) = \sum_{s} \hat{b}(s) \log(\hat{b}(s)).
\end{equation}

Though ${H}_{\hat{b}}(s)$ is known to be biased, \citet{Paninski} established some useful properties of it.
\begin{theorem} \label{th:entropyEst}
(Paninski 2003)
\begin{equation} \label{eq:probBound}
(a) \ \Pr(|{H}_{\hat{b}}(s) - \mathbb{E}[{H}_{{\hat{b}}}(s)~|~b]|\geq \eta) \leq \delta_{\eta}, 
\end{equation}
where $\delta_{\eta} = 2e^{\frac{-M}{2}\eta^{2}(\log(M))^{-2}}$.
\begin{equation} \label{eq:biasBound}
(b) \ \mu_{M}(b) \leq \mathbb{E}[{H}_{{\hat{b}}}(s)~|~b]- H_{{b}}(s)\leq 0,
\end{equation}
where $\mu_{M}(b) = -\log(1 + \frac{\psi_b(s) - 1}{M})$ and $\psi_{b}(s)$ is the support of $b(s)$. 
\end{theorem}

\noindent Hence \eqref{eq:probBound} bounds the variance of ${H}_{\hat{b}}(s)$ and \eqref{eq:biasBound} bounds its bias, which is always negative.

\subsection{Lower Confidence Bound}
Let ${H}_{\hat{b}}^{\mathcal{A}}(s|\mathbf{z})$ be defined as:
\begin{equation}
{H}_{\hat{b}}^{\mathcal{A}}(s|\mathbf{z}) = \sum_{\mathbf{z}_i \in \Omega} \Pr(\mathbf{z}_i|b,\mathcal{A}) {H}_{\hat{b}_{\mathbf{z}_{i}}^{\mathcal{A}}}(s),
\end{equation}
where ${H}_{\hat{b}_{\mathbf{z}_{i}}^{\mathcal{A}}}(s)$ is the MLE of the entropy of the posterior distribution $\hat{b}^{\mathcal{A}}_{\mathbf{z}}(s)$.

\begin{lemma}  \label{lem:low}
With probability $1 - \delta_{l}$, 
\begin{equation} 
{H}_{\hat{b}}^\mathcal{A}(s|\mathbf{z}) \leq H_{b}^\mathcal{A}(s|\mathbf{z}) + \eta,
\end{equation}
where $\delta_{l} =  |\Omega|\delta_{\eta}$.
\end{lemma}
\begin{proof}
For a given $\mathbf{z}_{i}$, \eqref{eq:probBound} and \eqref{eq:biasBound} imply that, with probability $1 - \delta_{\eta}$, 
\begin{equation}
{H}_{\hat{b}_{\mathbf{z}_i}^{\mathcal{A}}}(s) \leq H_{b_{\mathbf{z}_{i}}^{\mathcal{A}}}(s) + \eta
\end{equation}
This is true for each $\mathbf{z}_i \in \Omega$. Taking an expectation over $\mathbf{z}_{i}$ and using a union bound, yields the final result.
\end{proof}

Typically, the bottleneck in computing ${H}_{\hat{b}}^{\mathcal{A}}(s|\mathbf{z})$ is performing the belief update to find $\hat{b}_{\mathbf{z}_{i}}^\mathcal{A}$ for each $\mathbf{z}_i$.  In practice, we approximate these using \emph{particle belief updates} \cite{particlef}, which, for a given $\mathbf{z}_i$, generate a sample $s^j$ from $\hat{b}(s)$ and then an observation $\mathbf{z}'$ from $\Pr(\mathbf{z}|s^j,\mathcal{A})$.  If $\mathbf{z}_i = \mathbf{z}'$, then $s^j$ is added to the set of samples approximating $\hat{b}_{\mathbf{z}_{i}^\mathcal{A}}$.  Consequently, ${H}_{\hat{b}}^{\mathcal{A}}(s|\mathbf{z})$ can be tightened by increasing $M$, the number of samples used to estimate $\hat{b}(s)$, and/or increasing the number of samples used to estimate each $\hat{b}_{\mathbf{z}_{i}^\mathcal{A}}$. However, tightening ${H}_{\hat{b}}^{\mathcal{A}}(s|\mathbf{z})$ by using larger values of $M$ is not practical as computing it involves new posterior belief updates (with a larger value of $M$) and hence increases the computational cost of tightening ${H}_{\hat{b}}^{\mathcal{A}}(s|\mathbf{z})$. 

\subsection{Upper Confidence Bound}

Since ${H}_{\hat{b}}(s)$ is negatively biased, finding an upper confidence bound is more difficult.  A key insight is that such a bound can nonetheless be obtained by estimating posterior entropy using an artificially ``coarsened'' observation function. That is, we group all possible observations into a set ${\Phi}$ of clusters and then pretend that, instead of observing $\mathbf{z}$, the agent only observes what cluster $\mathbf{z}$ is in.  Since the observation now contains less information, the conditional entropy will be higher, yielding an upper bound.  Furthermore, since the agent only has to reason about $|{\Phi}|$ clusters instead of $|\Omega|$ observations, it is also cheaper to compute.  Any generic clustering approach, e.g., ignoring certain observation features can be used, though in some cases domain expertise may be exploited to select the clustering that yields the tightest bounds.

Let $\mathbf{r} = \langle r_1 \dots r_n \rangle$ represent a crude approximation of $\mathbf{z}$. That is, for every $i$, $r_{i}$ is obtained from $z_{i}$ by $r_i = f(z_i,d)$, where $f$ clusters $z_{i}$ into $d$ clusters 
deterministically and $r_i$ denotes the cluster to which $z_i$ belongs.  %
Also, if $z_i = \emptyset$, then $r_i = \emptyset$. Note that $H_{b}(\mathbf{r}|\mathbf{z}) = 0$ and the domain of $r_i$ and $r_j$ share only $\emptyset$ for all $i$ and $j$.

\begin{lemma}
$H_{b}^\mathcal{A}(s|\mathbf{z}) \leq H_{b}^\mathcal{A}({s}|\mathbf{r})$.
\end{lemma}
\begin{proof}
Using the chain rule for entropy, on $H_{b}^{\mathcal{A}}(s,\mathbf{z}|\mathbf{r})$
\begin{equation}
\begin{split}
H_{b}^\mathcal{A}(s|\mathbf{z,r}) + H_{b}^\mathcal{A}(\mathbf{z|r}) = H_{b}^\mathcal{A}(\mathbf{z}|s,\mathbf{r}) + H_{b}^\mathcal{A}(s|\mathbf{r}).
\end{split}
\end{equation}
Since $\mathbf{r}$ contains no additional information, $H_{b}^\mathcal{A}(s|\mathbf{z,r})$ = $H_{b}^\mathcal{A}(s|\mathbf{z})$, and $H_{b}^\mathcal{A}(s|\mathbf{z}) + H_{b}^\mathcal{A}(\mathbf{z|r}) = H_{b}^\mathcal{A}(\mathbf{z}|s,\mathbf{r}) + H_{b}^\mathcal{A}(s|\mathbf{r}).$
Since conditioning can never increase entropy \cite{cover}, $H_{b}^\mathcal{A}(\mathbf{z}|s,\mathbf{r}) \leq H_{b}^\mathcal{A}(\mathbf{z|r})$, and the stated result holds. \qedhere
\end{proof}

\noindent $H_{b}^\mathcal{A}(s|\mathbf{r})$ is cheaper to compute than $H_{b}^\mathcal{A}(s|\mathbf{z})$ because it requires only $|\Phi|$ belief updates instead of $|\Omega|$. Starting with a small $\Phi$, $H^{\mathcal{A}}_{\hat{b}}(s|\mathbf{r})$ can be tightened by increasing the number of clusters and thus $|\Phi|$.

Note that computing $H_{b}^{\mathcal{A}}(s|\mathbf{r})$ requires $\Pr(\mathbf{r}|s,\mathcal{A})$, which can be obtained by marginalizing $\mathbf{z}$ out from $\Pr(\mathbf{z}|s,\mathcal{A})$, a computationally expensive operation. However, this marginalization only needs to be done once and can be reused when performing greedy maximization for various $b(s)$. This occurs naturally in, e.g., sensor selection, where the hidden state that the agent wants to track evolves over time. At every time step, $b(s)$ changes and a new set $\mathcal{A}^{P}$ must be selected.

However, computing $H_{b_{\mathbf{r}}^{\mathcal{A}}}(s)$ still requires iterating across all values of $s$. Thus, to lower the computational cost further, we use estimates of entropy, as with the lower bound:
\begin{equation}
{H}_{\hat{b}}^\mathcal{A}(s|\mathbf{r}) = \sum_{\mathbf{r}_i \in {\Phi}} \Pr(\mathbf{r}_i|b,\mathcal{A}){H}_{\hat{b}_{\mathbf{r}_i}^{\mathcal{A}}}(s).
\end{equation}
Computing ${H}_{\hat{b}}^\mathcal{A}(s|\mathbf{r})$ is cheaper than ${H}_{b}^\mathcal{A}(s|\mathbf{r})$ but is not gauranteed to be greater than ${H}_{b}^\mathcal{A}(s|\mathbf{z})$ since the entropy estimates have negative
bias. However, we can still obtain an upper confidence bound.
\begin{lemma} \label{lem:up}
With probability $ 1 - \delta_{u}$
\begin{equation}
H_{b}^\mathcal{A}(s|\mathbf{z}) \leq {H}_{\hat{b}}^\mathcal{A}(s|\mathbf{r}) + \eta - \mu_{M}(b),
\end{equation}
where $\delta_{u} =  |\Phi|\delta_{\eta}$.
\end{lemma}
\begin{proof}
\eqref{eq:biasBound} implies that, for any fixed $\mathbf{r}_i \in {\Phi}$, 
\begin{equation}
H_{b_{\mathbf{r}_i}^{\mathcal{A}}}(s) \leq 
\mathbb{E}
[{H}_{\hat{b}_{\mathbf{r}_i}^{\mathcal{A}}}(s)~|~{b}_{\mathbf{r}_i}^{\mathcal{A}}] - \mu_{M}(b). 
\end{equation} 
Taking an expectation on both sides:
\begin{equation} \label{eq:temp}
\begin{split} 
\mathbb{E}_{\mathbf{r}_i}[H_{b_{\mathbf{r}_i}^{\mathcal{A}}}(s)~|~b,\mathcal{A}] \leq 
\mathbb{E}_{\mathbf{r}_i}
[\mathbb{E}
[{H}_{\hat{b}_{\mathbf{r}_i}^{\mathcal{A}}}(s)~|~{b}_{\mathbf{r}_i}^{\mathcal{A}}]~|~{b},\mathcal{A}] - \mu_{M}(b), \nonumber \\ 
\end{split}
\end{equation} 
Now, \eqref{eq:probBound} implies that, with probability $ 1 - \delta_{\eta}$,
\begin{equation}
\mathbb{E}
[{H}_{\hat{b}_{\mathbf{r}_i}^{\mathcal{A}}}(s)~|~{b}_{\mathbf{r}_i}^{\mathcal{A}}] 
\leq {H}_{\hat{b}_{\mathbf{r}_i}^{\mathcal{A}}}(s) + \eta. 
\end{equation}
Taking expectations on both sides and using a union bound gives, with probability $1 - \delta_u$,
\begin{equation}
\begin{split}
H_{b}^\mathcal{A}(s|\mathbf{r}) \leq {H}_{\hat{b}}^\mathcal{A}(s|\mathbf{r}) + \eta - \mu_{M}(b). 
\qedhere
\end{split}
\end{equation}
\end{proof}

In practice, we use a larger value of $M$ when computing ${H}_{\hat{b}}^{\mathcal{A}}(s|\mathbf{r})$ than ${H}_{\hat{b}}^{\mathcal{A}}(s|\mathbf{z})$.  Doing so is critical for reducing the negative bias in ${H}_{\hat{b}}^{\mathcal{A}}(s|\mathbf{z})$.  Furthermore, doing so does not lead to intractability because choosing a small $|\Phi|$ ensures that few belief updates will be performed.  

Thus, when computing ${H}_{\hat{b}}^{\mathcal{A}}(s|\mathbf{z})$, we set $M$ low but perform many belief updates; when computing ${H}_{\hat{b}}^{\mathcal{A}}(s|\mathbf{r})$ we set $M$ high but perform few belief updates.  This yields cheap upper and lower confidence bound for conditional entropy.

The following theorem ties together all the results presented in this paper. Note that, since $F$ is defined as \emph{negative} conditional entropy, $L$ is defined using our \emph{upper} bound and $U$ using our \emph{lower} bound.

\begin{theorem}
Let $F(\mathcal{A}) = H_{b}(s) - H_{b}^{\mathcal{A}}(s|\mathbf{z})$. Let $\mathtt{tighten}$ be defined such that $\mathtt{tighten(\mathcal{A}, t)}$ returns \\ $U_{t}(\mathcal{A}) = H_{b}(s) - H_{\hat{b}}^{\mathcal{A}}(s|\mathbf{z}) + \sqrt{\frac{2 (\log(M))^2}{M} \log(\frac{ 2n |\Omega| t (t + 1)}{\delta_u})} $ and \\
 $L_{t}(\mathcal{A}) = H_{b}(s) - [H_{\hat{b}}^{\mathcal{A}}(s|\mathbf{r}) +  \sqrt{\frac{2(\log(M))^{2}}{M} \log(\frac{2 n|\Omega| t (t+1)}{ \delta_{l}}}) + \log(1 + \frac{1}{M}{(|supp(b)| - 1)})]$. Let $\mathcal{A}^P = \mathtt{pac}\mathhyphen\mathtt{greedy}\mathhyphen\mathtt{max}(\mathtt{tighten},\mathcal{X},k,\epsilon_1)$ and $\mathcal{A}^* = \arg\max_{\mathcal{A} \in \mathcal{A}^+} F(\mathcal{A})$, where $\mathcal{X} = \{1, 2, \dots, n\}$ and $\mathcal{A}^+ = \{ \mathcal{A} \subseteq \mathcal{X}: |\mathcal{A}| \leq k \}$. If $\mathbf{z}$ is conditionally independent given $s$ then,
with probability $1 - \delta$, 
\begin{equation}
F(\mathcal{A}^{P}) \geq (1 - e^{-1})F(\mathcal{A}^{*})  - \epsilon,
\end{equation}
where $\delta = k(\delta_l + \delta_u)$, $\epsilon = k\epsilon_1$.
\end{theorem}
\begin{proof}
We showed that with probability $1 - \frac{\delta_{l}}{nt(t+1)}$, $L_{t}(\mathcal{A}) \leq F(\mathcal{A})$ and with probability $1 - \frac{\delta_u}{tn(t+1)}$, $U_{t}(\mathcal{A}) \geq F(\mathcal{A})$. 
Krause and Guestrin, 2005 showed that $Q$ is non-negative, monotone and submodular if $\mathbf{z}$ is conditionally independent given $s$. The $\mathtt{tighten}$ procedure can be designed by tightening the upper and lower bounds by either increasing $M$ or by changing the clusters used to estimate $H_{\hat{b}}(s|\mathbf{r})$. Thus, Theorem \ref{th:main1} with $\epsilon = k\epsilon_1$ and $\delta_{1} = \delta_{u} + \delta_{l}$
implies the stated result.
\end{proof}

\section{Related Work}
Most work on submodular function maximization focuses on algorithms for approximate greedy maximization that minimize the number of evaluations of $F$ \cite{lazy,multistage,thresholdgreedy,lazier}. In particular, \citet{lazier} randomly sample a subset from $\mathcal{X}$ on each iteration and select the element from this subset that maximizes the marginal gain. \citet{thresholdgreedy} selects an element on each iteration whose marginal gain exceeds a certain threshold. Other proposed methods that maximize surrogate submodular functions \cite{multistage,activeLearn1} or address streaming \cite{gomes} or distributed settings \cite{lazierdist}, also assume access to exact $F$.  In contrast, our approach assumes that $F$ is too expensive to compute even once and works instead with confidence bounds on $F$. \citet{igkrause} propose approximating conditional entropy for submodular function maximization while still assuming they can compute the exact posterior entropies; we assume computing exact posterior entropy is prohibitively expensive.

\citet{streeter2009online} and \citet{diverse} propose conceptually related methods that also assume $F$ is never computed exactly.  However, their \emph{online} setting is fundamentally different in that the system must first select an entire subset $\mathcal{A} \in \mathcal{A}^+$ and only then receives an estimate of $F(\mathcal{A})$, as well as estimates of the marginal gain of the elements in $\mathcal{A}$.  Since the system learns over time how to maximize $F$, it is a variation on the multi-armed bandit setting.  By contrast, we assume that feedback about a given element's marginal gain is available (through tightening $U$ and $L$) \emph{before} committing to that element.

As mentioned earlier, Algorithm \ref{alg:pac-m} is closely related to \emph{best arm identification} algorithms \cite{bestarm}. However, such methods assume an unbiased estimator of $F$ is available and hence concentration inequalities like Hoeffding's inequality are applicable. An exception is the work of \citet{racing}, which bounds the difference between an entropy estimate and that estimate's expected value.  However, since the entropy estimator is biased, this does not yield confidence bounds with respect to the true entropy. While they propose using their bounds for best arm identification, no guarantees are provided, and would be hard to obtain since the bias in estimating entropy has not been addressed. However, their bounds \cite[Corollary 2]{racing} could be used in place of Theorem \ref{th:entropyEst}a. While other work proposes more accurate estimators for entropy \cite{IGestimates,Paninski,Schurmann}, they are not computationally efficient and thus not directly useful in our setting.

Finally, greedy maximization is known to be \emph{robust to noise}  \cite{streeter2009online,submodularity}: if instead of selecting $i^{G} = \argmax_{i \in \mathcal{X} \setminus \mathcal{A}^{G}}\Delta(i|\mathcal{A}^G)$, we selects $i'$ such that $\Delta(i'|\mathcal{A}^{G}) \geq \Delta(i^G|\mathcal{A}^{G}) - \epsilon_1$, the total error is bounded by $\epsilon = k \epsilon_1$. We exploit this property in our method but use confidence bounds to introduce a probabilistic element, such that with high probability $\Delta(i^{P}|\mathcal{A}^{G}) \geq \Delta(i^G|\mathcal{A}^{G}) - \epsilon_1$.

\section{Experimental Results} \label{sec:exp}

We evaluated PAC greedy maximization on the problem of tracking multiple people using a multi-camera system. The problem was extracted from a real-world dataset collected over 4 hours using 13 CCTV cameras located in a shopping mall. Each camera uses a \emph{FPDW} pedestrian detector \cite{dollar} to detect people in each camera image and \emph{in-camera tracking} \cite{Bouma} to generate tracks of the detected people's movement over time.  The dataset thus consists of 9915 trajectories, each specifying one person's $x$-$y$ position. The field of view of a few cameras were divided into two or three separate regions and each region was treated as an independent camera, so as to enable more challenging experiments with as many as $n = 20$ cameras.

We first consider tracking a single person. The hidden state $s$ is modeled as the position and velocity of the person and described by the tuple $\langle x,y,v_x,v_y \rangle$, where $x$ and $y$ describe the position and $v_x$ and $v_y$ describe his velocity in the $x$ and $y$ directions.  Both $x$ and $y$ are integers in $\{0,\ldots,150\}$.  The surveillance area can be observed with $n=20$ cameras and, if selected, each camera produces an observation $\langle z^{x}, z^{y} \rangle$ containing an estimate of the person's $x$-$y$ position. 

We assume a person's motion in the $x$ direction is independent of his motion in the $y$ direction. Given the current position $x_{curr}$, the future position $x_{next}$ is a deterministic function of $x_{curr}$ and the current velocity in $x$-direction $v_{x}^{curr}$, i.e., $x_{next} = x_{curr} + v_{x}^{curr}$. The same is true for the $y$ position. The future velocity $v_{next}$ is modeled as a Gaussian distribution with the current velocity as the mean and the standard deviation, which depends on the current $x$-$y$ position, learnt from the data, i.e., $v_x^{next} \sim \mathcal{N}(v_x^{curr}, \sigma^x)$ and $v_y^{next} \sim \mathcal{N}(v_y^{curr}, \sigma^y)$. The observations are assumed to be conditionally independent given the state and are generated from a Gaussian distribution with the true position as the mean and a randomly generated standard deviation. Since ground truth data about people's locations is not available, learning the standard deviation is not possible. A belief $b(s)$ about the person's location was maintained using an unweighted particle filter with 200 particles. Given a subset of the sensors and the observations they generated, $b(s)$ is updated using a particle belief update \cite{particlef}

To evaluate a given algorithm, a trajectory was sampled randomly. At each timestep in the trajectory, a subset of $k$ cameras out of $n=20$ were selected by the algorithm. Using the resulting observations, the person was tracked using an unweighted particle filter \cite{particlef}, starting from a random initial belief. At each timestep, a prediction $\argmax_s b(s)$ about the person's location was compared to the person's true location.  Performance is the total number of correct predictions made over multiple trajectories. For multi-person tracking, the best subsets of cameras for each person were computed independently of each other and then the subset with the highest value of $F$ was selected.

\begin{figure}[h!]
\centering
\includegraphics[scale=0.18]{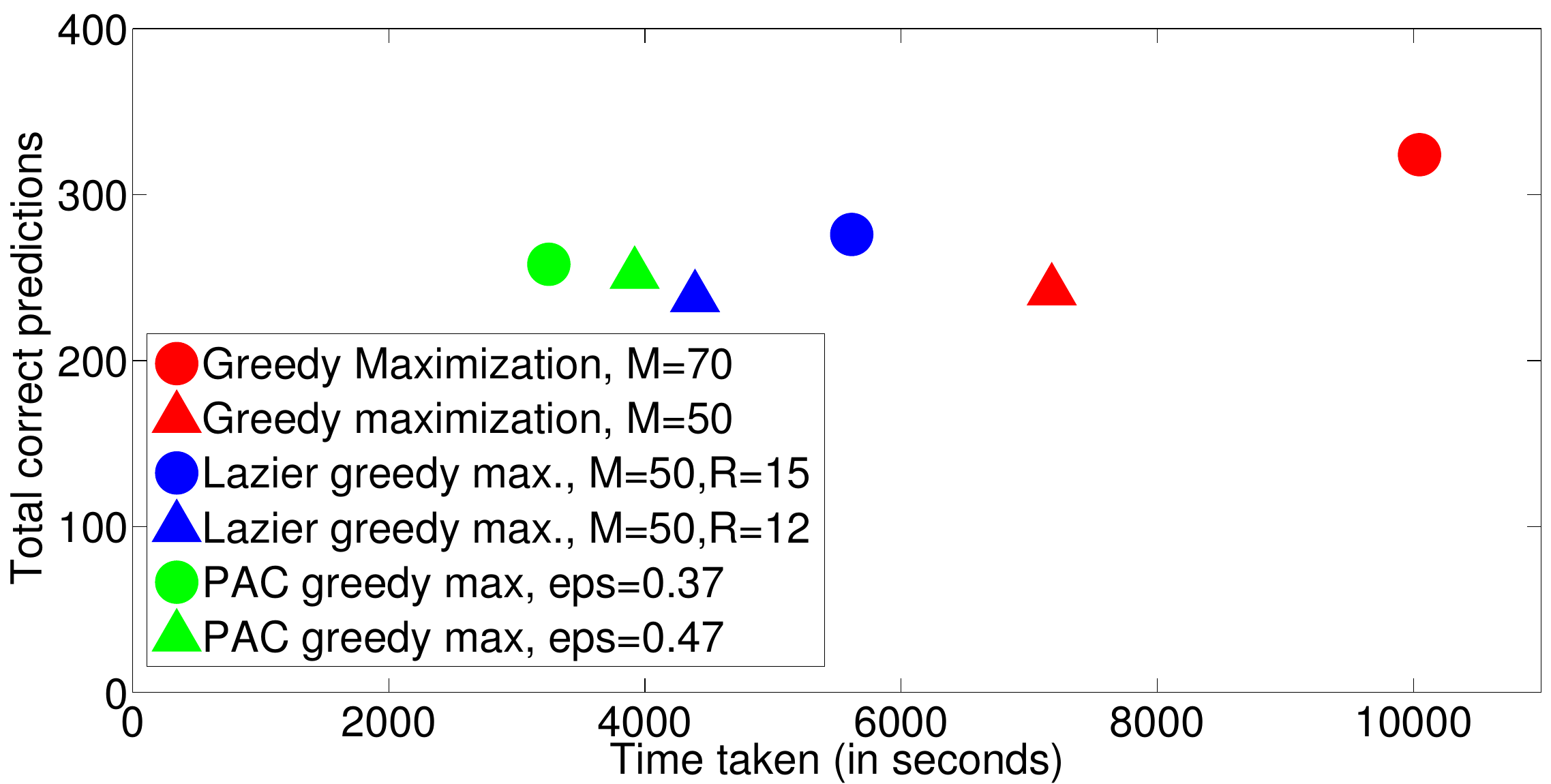} 
\includegraphics[scale=0.18]{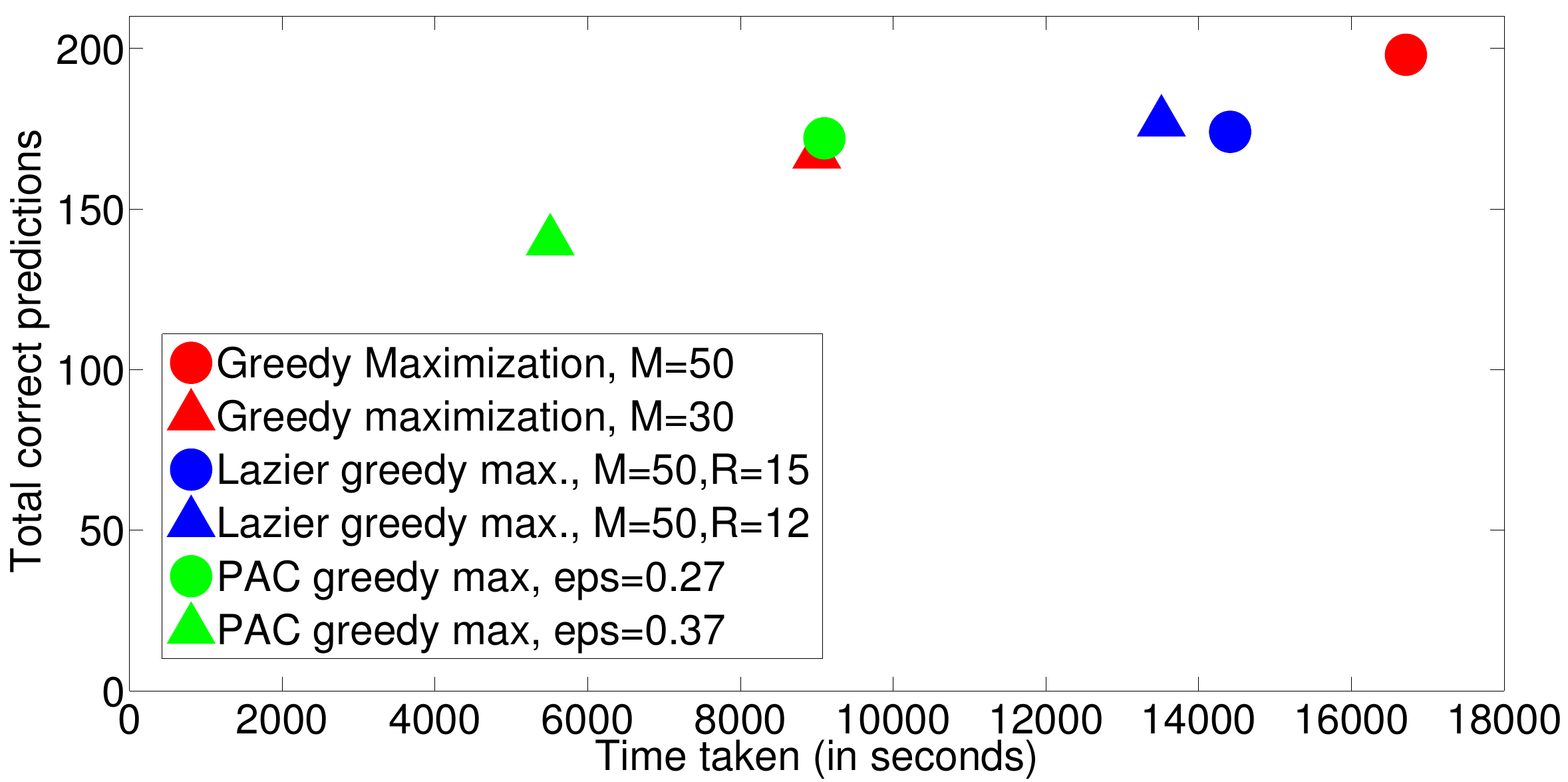} 
\includegraphics[scale=0.18]{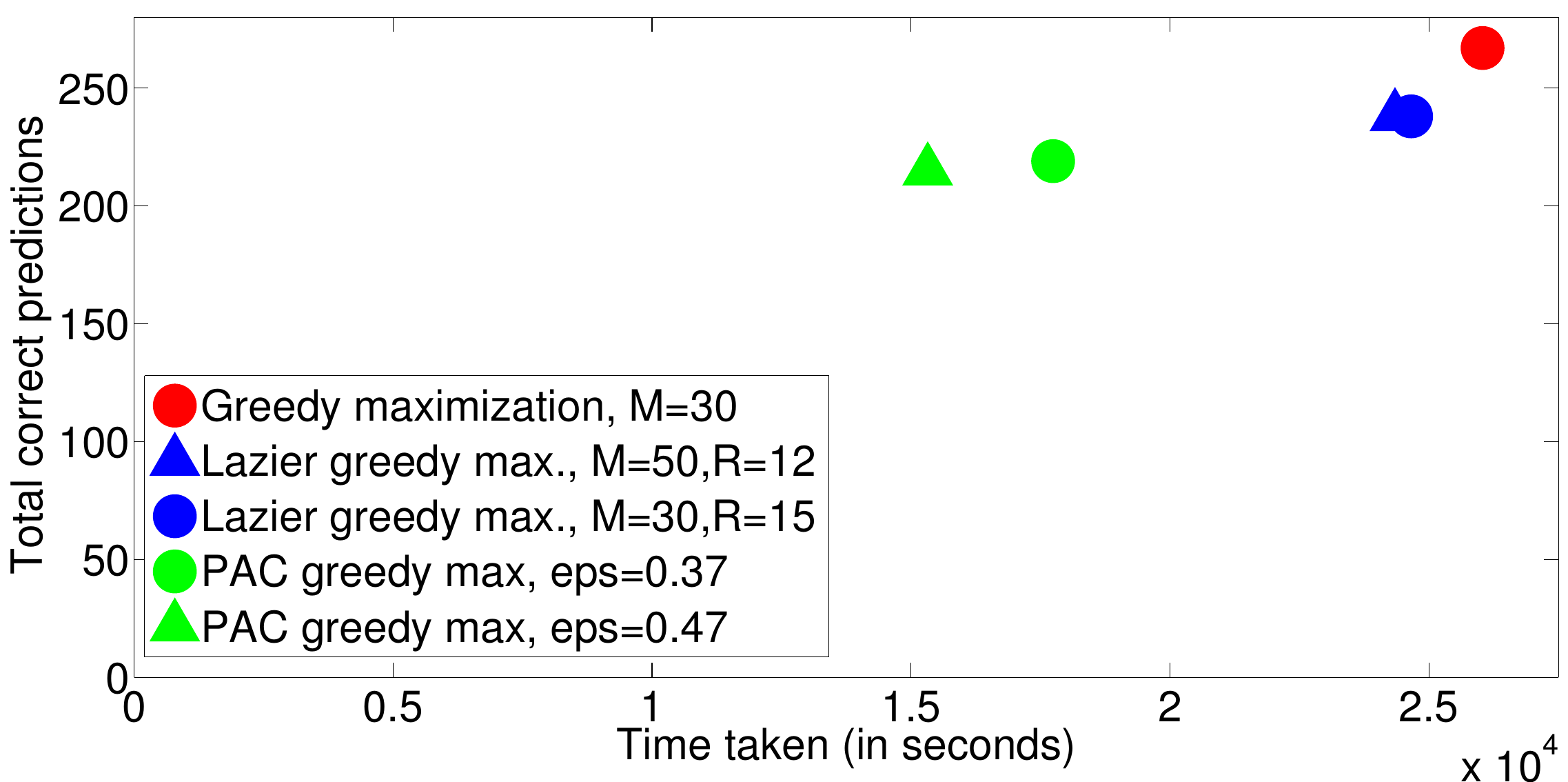}
\caption{Multi-person tracking for $n=20$ and (top) $k=1$; (middle) $k=2$; (bottom) $k=3$.}
\label{fig:results}
\end{figure}

We conducted experiments with different values of $n$ and $k$. As a baseline, we use greedy maximization and lazier greedy maximization. Since we cannot compute $F$ exactly, greedy maximization simply uses an approximation, based on MLE estimates of conditional entropy, ignoring the resulting bias and making no attempt to reason about confidence bounds. Lazier greedy maximization, in each iteration, samples a subset of size $R$ from $\mathcal{X}$ and selects from that subset the element that maximizes the estimated marginal gain. Neither greedy nor lazier greedy maximization employ lazy evaluations, i.e., pruning elements via a priority queue as in lazy greedy maximization, because the reliance on approximation of $F$ means pruning is no longer justified. In addition, since lazy greedy maximization's pruning is based on marginal gain instead of $F$, the bias is exacerbated by the presence of two entropy approximation instead of one.

For greedy maximization and lazier greedy maximization, the number of samples $M$ used to approximately compute $F$ was varied from 10 to 100. For lazier greedy maximization, the value of $R$ was also varied from 5 to 15. For PAC greedy maximization, the number of samples used to compute the upper and lower bounds were fixed to 20 and 10 respectively, while the parameter $\epsilon_{1}$ was varied from 0.1 to 0.5. On average the length of each trajectory sampled was 30 timesteps and the experiments were performed on 30 trajectories for $k=1$, 17 trajectories for $k=2$ and 20 trajectories for $k=3$, with 5 independent runs for $k=3$ and 3 independent runs for $k=2$ and $k=1$. To avoid clutter, we show results for only the two best performing parameter settings of each algorithm.

Figure \ref{fig:results} shows the number of correct predictions ($y$-axis) against the runtime ($x$-axis) of each method at various settings of $M$, $R$ and $\epsilon$.
Thus, the top left is the most desirable region.
In general, PAC greedy maximization performs nearly as well as the best-performing algorithm but does so at lower computational cost. Naively decreasing the number of samples only worsens performance and does not scale with $k$ as the computational cost of even performing greedy maximization with nominal samples is huge in the bottom plot. PAC greedy maximization on the other hand performs well in all the three settings and scales much better as $k$ increases, making it more suitable for real-world problems.

\section{Conclusions \& Future Work}
This paper proposed PAC greedy maximization, a new algorithm for maximizing a submodular function $F$ when computing $F$ exactly is prohibitively expensive. Our method assumes access to cheap confidence bounds on $F$ and uses them to prune elements on each iteration.  When $F$ involves entropy, as is common in many applications, obtaining confidence bounds is complicated by the fact that no unbiased estimator of entropy exists.  Therefore, we also proposed novel, cheap confidence bounds on conditional entropy that are suitable for use by PAC greedy maximization. We proved that the resulting method has bounded error with high probability. Our empirical results demonstrated that our approach performs comparably to greedy and lazier greedy maximization, but at a fraction of the computational cost, leading to much better scalability. In future  work, we aim to develop good strategies for clustering observations to obtain tight upper confidence bounds on conditional entropy, and to combine our upper and lower confidence bounds with more sophisticated best-arm identification algorithm sto produce an even more efficient version of PAC greedy maximization.

\section*{Acknowledgments}
We thank Henri Bouma and TNO for providing us with the dataset used in our experiments.  We also thank the STW User Committee for its advice regarding active perception for multi-camera tracking systems. This research is supported by the Dutch Technology Foundation STW (project \#12622), which is part of the Netherlands Organisation for Scientific Research (NWO), and which is partly funded by the Ministry of Economic Affairs.
Frans Oliehoek is funded by NWO Innovational Research Incentives Scheme Veni \#639.021.336.

\bibliographystyle{named}
\bibliography{tripleGreedy}

\end{document}